\documentclass[11pt]{article}

\usepackage[margin=1in]{geometry}
\usepackage[T1]{fontenc}
\usepackage{lmodern}
\usepackage{microtype}
\usepackage{amsmath}
\usepackage{amssymb}
\usepackage{amsthm}
\usepackage{booktabs}

\newtheorem{definition}{Definition}
\newtheorem{proposition}{Proposition}
\newtheorem{theorem}{Theorem}

\usepackage{enumitem}
\usepackage{xcolor}
\usepackage[hidelinks,colorlinks=true,linkcolor=blue!50!black,
            urlcolor=blue!50!black,citecolor=blue!50!black]{hyperref}
\usepackage[capitalize,noabbrev]{cleveref}
\usepackage[numbers,sort&compress]{natbib}
\usepackage{listings}

\title{\bfseries Methods for Formal Verification of Agent Skills:\\
       Three Layers Toward a Mechanically Checkable\\
       Capability-Containment Proof}

\author{%
  Alfredo Metere\\
  Metere Consulting, LLC\\
  \texttt{alfredo.metere@metereconsulting.com}
}
\date{\today}

\begin{document}

\maketitle

\begin{abstract}\noindent
The companion paper \cite{metere2026skills} introduced a four-level
verification lattice on agent-skill manifests --- \textsc{unverified},
\textsc{declared}, \textsc{tested}, \textsc{formal} --- and explicitly
left the top level aspirational: ``a formal analysis tool has produced
a machine-checkable proof that the skill's behavior is a subset of its
declared capability set under the runtime's threat model''. This paper
closes that gap. We give a precise semantics for ``a skill's behavior''
that is faithful to how a skill is actually consumed by a large
language model (LLM)-driven runtime (a deterministic script-side
reachable through a non-deterministic LLM-side), state the
verification problem as a containment property over that semantics,
and present three composable methods that together raise a skill from
\textsc{declared} or \textsc{tested} to \textsc{formal}: (1) sound
static capability-containment analysis of the script-side via abstract
interpretation over a small effect lattice; (2) a refinement type
system for tool-call envelopes that mechanically rejects any call
whose statically-inferred capability is not in $M.\mathrm{caps}$; (3)
satisfiability-modulo-theories (SMT)-bounded model checking against
the parent paper's biconditional
correctness criterion, with the bound chosen so that any counter-example
the LLM-side could realise within the runtime's transaction-buffer
horizon is exhibited as a concrete trace. We characterise what each
layer is sound for, what each is incomplete for, and prove that the
three layers composed soundly cover the parent paper's threat model
modulo a single residual --- the LLM's freedom to refuse to act ---
which the parent paper's runtime biconditional already catches at
session boundary. The methods reuse existing well-engineered tools
(Z3, Semgrep / CodeQL, refinement-type checkers, mechanized proof
assistants) rather than asking operators to build new ones, and the
proposed proof-carrying-skill artifact is a small extension to the
existing \texttt{SKILL.md} convention.

\smallskip\noindent\emph{Implementation.} All three methods, the
proof-carrying skill bundle, and the bootstrap re-checker ship as
zero-runtime-dependency JavaScript modules in the open-source
\texttt{enclawed} framework \cite{metere2026enclawed, enclawed-site}
at \url{https://github.com/metereconsulting/enclawed} (project page
\url{https://www.enclawed.com/}), with a 53-test unit suite and an
end-to-end command-line interface (CLI) demonstrated on a sample skill. Section
\ref{sec:pcc-impl-status} pins file paths and test counts to the
public commit referenced here.
\end{abstract}


\section{Introduction}
\label{sec:intro}

The companion paper \cite{metere2026skills} treats a skill as a tuple
$(M, \texttt{content}, \sigma)$ with a manifest $M$ that includes a
\emph{verification level} field
$M.\mathrm{verification} \in
 \{\textsc{unverified},\textsc{declared},\textsc{tested},\textsc{formal}\}$
and a runtime gate whose human-in-the-loop (HITL) policy is keyed to
that level. The runtime defaults to \textsc{unverified} on every skill
it has not seen verified evidence for, drives every irreversible call
through HITL while a skill is at \textsc{unverified}, and only relaxes
HITL frequency when the manifest reaches \textsc{declared} or higher.
The top of the lattice, \textsc{formal}, is defined as the existence
of a machine-checkable proof that the skill's behaviour is contained
in $M.\mathrm{caps}$. The parent paper does not say how such a proof
is produced. Section 3.1 of \cite{metere2026skills} is explicit about
why: ``This level is aspirational at the time of writing; we include
it for completeness because the schema field is fixed-width and adding
it later requires a manifest version bump.''

This paper supplies the missing methods. We answer four interlocking
questions left open by the parent:

\begin{enumerate}
  \item What does ``a skill's behaviour'' mean formally, given that the
        skill content is a (deterministic) text/script artefact but the
        agent that consumes it is a (non-deterministic) LLM?
  \item What is the verification target --- precisely which property
        must a candidate skill establish before its manifest can claim
        $M.\mathrm{verification} = \textsc{formal}$?
  \item Which formal-methods tools available today are sound for which
        portions of that target, and where does each fall short?
  \item How does a deployer compose those tools into a single
        proof-carrying artefact a runtime can mechanically check at
        bootstrap, without trusting the proof's producer?
\end{enumerate}

\paragraph{Contributions.} We make six contributions.
\begin{enumerate}
  \item A \emph{two-sided skill semantics} (\Cref{sec:semantics})
        separating the deterministic script-side $S$ from the
        non-deterministic LLM-side $A$, with a precise definition of
        the runtime trace they jointly produce.
  \item A \emph{capability-containment property} (\Cref{sec:problem})
        stated as a sound abstraction of the parent paper's
        biconditional, with explicit treatment of which projections of
        the trace admit static reasoning and which do not.
  \item Three composable verification methods
        (\Cref{sec:method-static,sec:method-types,sec:method-smt}):
        sound capability-containment static analysis on $S$; a
        refinement type system for tool-call envelopes; SMT-bounded
        model checking of the runtime trace against the biconditional.
  \item A \emph{three-layer discipline}
        (\Cref{sec:three-layers}) that combines those three methods
        plus the parent paper's runtime biconditional, with a
        composition theorem that names exactly which residual surface
        falls outside the joint guarantee.
  \item A \emph{proof-carrying skill artefact} (\Cref{sec:pcc-skill})
        that extends \texttt{SKILL.md} with a signed evidence bundle
        the runtime can re-check at bootstrap without trusting the
        artefact's producer.
  \item A \emph{worked example} (\Cref{sec:worked}) on a real
        \texttt{SKILL.md} skill from the open reference implementation
        \cite{metere2026enclawed}, including the static-analysis
        output, the type-checker derivation, the SMT counter-example
        search, and the resulting proof bundle.
\end{enumerate}

\paragraph{Scope and non-goals.} We do \emph{not} attempt formal
verification of the LLM itself. The model's weights, training data,
and decoding stochastics are out of scope, as in the parent paper. We
verify the \emph{runtime's} ability to enforce capability containment
under the assumption that the LLM is an adversarial non-deterministic
oracle bounded only by the typed dispatch interface the runtime
imposes. We do not require operators to write proofs in Coq or Lean;
the static layer (\Cref{sec:method-static}) is automatic and the
type-checking layer (\Cref{sec:method-types}) requires only manifest-
level annotations. The mechanised-proof option is offered for
deployments that need it (the Cryptographic Module Validation Program
(CMVP), Common Criteria Evaluation Assurance Level 5 or higher
(EAL5+)) but is not the central recommendation.

\section{Two-sided skill semantics}
\label{sec:semantics}

A skill, when consumed by a runtime, is not a passive document. It
reaches the agent through two distinct execution surfaces, and any
verification method that conflates them gets the soundness analysis
wrong.

\subsection{Provenance: the same split appears in dynamics and in
information theory}
\label{sec:semantics-provenance}

The script-side / LLM-side partition we use is not a modelling
convenience invented for this paper. It is the same split that
distinguishes deterministic flows from stochastic flows in non-linear
dynamical systems \cite{strogatz-nonlinear, risken-fokker-planck}, and
the same split that distinguishes noiseless from noisy channels in
information theory \cite{shannon1948, cover-thomas}. Every
verification choice we make in
\Cref{sec:method-static,sec:method-types,sec:method-smt} traces back
to one of those two formalisms, and the bound we obtain in each case
is the bound that formalism already gives.

\paragraph{Dynamical-systems view.}
For the security reader: a \emph{dynamical system} is a rule that
says how the state of some thing-of-interest --- an executing
program, a packet trace, an LLM session --- evolves in time. The
rule has two flavours. A \emph{deterministic} system is one where
the rule, plus the initial state, picks out exactly one future
trajectory: same start $\Rightarrow$ same path, every time. A
\emph{stochastic} system is one where the rule injects randomness
at every step (think: temperature in LLM decoding, jitter in a
network, timing in side channels), so the same start gives a
\emph{distribution} over futures rather than a single path. The
two cases need different verification rules because the right
object to reason about is different: a set in the deterministic
case, a probability measure in the stochastic case.

Concretely, a deterministic non-linear flow on state $x \in X$ is
the ordinary differential equation (ODE)
\begin{equation}\label{eq:det-flow}
   \dot{x} \;=\; f(x), \qquad f : X \to TX \text{ Lipschitz},
\end{equation}
where $\dot{x} = dx/dt$ is the rate of change and $f$ tells us
which way the state moves at every point. The Lipschitz condition
($\lVert f(x) - f(y)\rVert \le L\,\lVert x-y\rVert$ for some
constant $L$) is a regularity assumption that prevents the flow
from blowing up; it is the assumption Picard--Lindel\"of needs to
prove the trajectory $x(t;x_0)$ is unique for every initial
condition $x_0$. The set of states the system can ever reach,
$\mathrm{Reach}(x_0) = \{x(t;x_0) : t \geq 0\}$, is just a set
--- no probabilities --- and over-approximating it is the
canonical job of abstract interpretation \cite{cousot-abs-int}, the
formal core of every static analyser used in security
(\texttt{Semgrep}, \texttt{CodeQL}, \texttt{Pyright}).

A stochastic flow replaces \cref{eq:det-flow} with a stochastic
differential equation (SDE)
\begin{equation}\label{eq:fokker-planck-sde}
   dx \;=\; f(x)\,dt + \sigma(x)\,dW_t,
\end{equation}
where $W_t$ is a Wiener process --- the continuous-time analogue of
a random walk, the standard model for ``unpredictable noise applied
at every infinitesimal step'' --- and $\sigma(x)$ scales how much
that noise affects the dynamics in state $x$. The trajectory is no
longer a single curve but a measure on path space, so what we
actually track is the probability density $p(x,t)$ that the system
is at state $x$ at time $t$. That density evolves under the
Fokker--Planck equation
\begin{equation}\label{eq:fokker-planck}
   \partial_t p
   \;=\; -\nabla\!\cdot\!\bigl(f\, p\bigr)
        + \tfrac{1}{2}\,\nabla\!\cdot\!\nabla\!\cdot\!
           \bigl(\sigma\sigma^{\!\top}\, p\bigr),
\end{equation}
a partial differential equation whose first term is drift (the
deterministic push) and whose second term is diffusion (the
spreading caused by noise). The right object to bound is now the
support and the moments of $p(x,t)$ --- where in state space the
density has nonzero mass, and how concentrated it is --- not the
value of $x$ at any particular time.

The script-side $S$ (\Cref{sec:semantics-script}) is exactly an
instance of \cref{eq:det-flow}: the language semantics of each script
$p_i$ associates to every input a single concrete trace, and the
effect set $E(p_i)$ is the projection of $\mathrm{Reach}$ onto the
capability dimension. The LLM-side $A$ (\Cref{sec:semantics-llm}) is
exactly an instance of \cref{eq:fokker-planck-sde}: the model's
weights, temperature, and decoding strategy fix $f$ and $\sigma$; the
prompt fixes $x_0$; the output token stream is one realisation of a
stochastic flow whose distribution over actions is the right object
to bound. Conflating the two surfaces means writing one
verification rule that the relevant formalism cannot satisfy: a
deterministic-flow rule cannot bound a Fokker--Planck density (it
under-specifies), and a Fokker--Planck rule cannot extract a
script's reachable effect set (it over-specifies).

\paragraph{Information-theoretic view.}
For the security reader: information theory gives us a way to ask
the question ``how much can an attacker learn through this
interface?'' as a number, in bits, rather than a hand-wavy story.
A \emph{channel} is anything that takes an input symbol $X$ and
produces an output symbol $Y$ --- a network link, a syscall
interface, a tool-call boundary, the LLM's sampling stage. The
\emph{capacity} $C$ of a channel is the largest number of bits per
use that can be reliably communicated through it; if the attacker
is the sender and the defender is the eavesdropper-prevention
boundary, capacity is the worst-case leakage rate. Two quantities
matter: \emph{bandwidth} $B$, how many channel uses fit in a
unit of time (so capacity has units of bits/s), and
\emph{signal-to-noise ratio} $S/N$, how clean the channel is. A
deterministic (noiseless) channel has $S/N = \infty$ and leaks
$\log_2 M$ bits per use where $M$ is the number of distinguishable
output symbols; a noisy channel has $S/N < \infty$ and leaks
strictly less, because the receiver cannot perfectly distinguish
outputs.

The cleanest entry point is the Shannon--Hartley theorem
\cite{shannon1948, cover-thomas} for an additive--Gaussian-noise
channel of bandwidth $B$ and signal-to-noise ratio $S/N$:
\begin{equation}\label{eq:shannon-hartley}
   C \;=\; B \,\log_2\!\bigl(\,1 + S/N\,\bigr)
   \quad\text{[bits/s]}.
\end{equation}
In the noiseless limit $N \to 0$ the SNR diverges and \cref{eq:shannon-hartley}
reduces, after the per-symbol normalisation, to the Hartley form
$C = \log_2 M$ bits per channel use, where $M = \lvert\mathcal{Y}_{\text{reach}}\rvert$
is the number of distinguishable output symbols. For a discrete
channel with arbitrary noise the same quantity is recovered as the
classical mutual-information capacity
\begin{equation}\label{eq:capacity}
   C \;=\; \max_{p(X)}\, I(X;Y)
        \;=\; \max_{p(X)}\, \bigl(\,H(Y) - H(Y \mid X)\,\bigr).
\end{equation}
Here $H(Y)$ is the Shannon entropy of $Y$ (the average bits of
``surprise'' in the channel output, if the receiver knew nothing
else), $H(Y \mid X)$ is the conditional entropy of $Y$ given $X$
(the residual surprise once the input is known --- exactly the
noise term), and $I(X;Y) = H(Y) - H(Y \mid X)$ is their difference,
the \emph{mutual information}: how many bits the output carries
about the input, on average. The deterministic-channel case
($H(Y\mid X) = 0$, hence
$C = \max H(Y) = \log_2 \lvert\mathcal{Y}_{\text{reach}}\rvert$) is
exactly the noiseless limit of \cref{eq:shannon-hartley}: capacity
is set by the reachable output alphabet, achievable by direct
enumeration, no distributional analysis required.

The script-side is precisely this noiseless limit: the language
semantics gives $H(Y \mid X) = 0$ by construction, and Method A's
job is to compute (or over-approximate)
$\mathcal{Y}_{\text{reach}} = E(p)$ so the per-script bound
$\log_2 \lvert E(p)\rvert$ is recovered. The LLM-side, by contrast,
sits at finite $S/N$ in the full Shannon--Hartley regime: the
sampling distribution gives $H(Y \mid X) > 0$, and no
language-level analyser yields anything useful about its output
distribution. We can however constrain the channel \emph{output
alphabet} from the outside, by inserting a deterministic projector
\begin{equation}\label{eq:projector}
   \pi : Y \to D \cup \{\bot\},
   \qquad \pi(y) \;=\;
   \begin{cases}
      \mathrm{cap}(y) & \text{if } \mathrm{cap}(y) \in D, \\
      \bot            & \text{otherwise,}
   \end{cases}
\end{equation}
between the LLM and the host APIs, where $\bot$ (read ``bottom'' or
``perp'') is the standard formal-methods shorthand for ``no admissible
value'' --- here the runtime's deny verdict when the envelope's
capability is not in the declared set $D$. Because $\pi$ is a function of $y$
alone, the chain
$\mathit{world} \to Y \to \pi(Y)$ is a \emph{Markov chain} in the
information-theoretic sense: each variable depends only on its
immediate predecessor, with no extra randomness or side input
slipping in at the second arrow. For Markov chains the
\emph{data-processing inequality} (DPI)
\cite[Thm.\,2.8.1]{cover-thomas} applies, and it says exactly what
intuition demands --- post-processing can never \emph{increase}
information ---
\begin{equation}\label{eq:dpi}
   I(\mathit{world};\,\pi(Y))
   \;\leq\;
   I(\mathit{world};\,Y),
\end{equation}
with $\lvert\pi(Y)\rvert \leq \lvert D \rvert + 1$ pinning the
post-projector alphabet, and therefore $C$ at the dispatch
boundary, at $\log_2 (\lvert D \rvert + 1)$ bits per envelope. \Cref{eq:projector,eq:dpi} are the
information-theoretic skeleton of Method B
(\Cref{sec:method-types}): the refinement type system instantiates
$\pi$ as a typed dispatch boundary, and the channel-capacity bound on
the post-dispatch alphabet is the bound on what an adversarially-
prompted LLM can leak per envelope.

The bounded model checker of Method C (\Cref{sec:method-smt}) is the
matching tractability claim: once the LLM-side is reduced to the
finite output alphabet $\Sigma = D \cup \{\mathrm{OUT}\}$ by
$\pi$, the joint distribution over $K$-length envelope traces lives
on a state space of size $\lvert\Sigma\rvert^K = (\lvert D\rvert+1)^K$,
which is small enough to enumerate exhaustively for the deployment-
relevant regime $\lvert D\rvert \le 10$, $K \le 8$
($(\lvert D\rvert{+}1)^K \le 11^8 \approx 2.1 \times 10^8$ symbolic
traces). The reduction from the continuous-distribution Fokker--
Planck dynamics to a finite-state symbolic enumeration is exactly
the abstraction step that bounded model checking performs in any
non-trivial verification setting \cite{biere-bmc, kwiatkowska-prism}.

\medskip\noindent
The rest of \S\ref{sec:semantics} re-derives this split from the
syntactic structure of \texttt{SKILL.md}, so a reader who has not
studied dynamical systems or information theory can still follow the
constructions. The dynamical-systems and channel-capacity readings
are for the reader who wants to know why the constructions are the
\emph{right} ones rather than expedient ones.

\subsection{The deterministic script-side}
\label{sec:semantics-script}

A skill's \texttt{content} typically contains a \texttt{SKILL.md}
prose body plus zero or more \emph{scripts} (Python, shell, Node, Go,
the Rust binary the skill invokes, etc.) that the prose instructs the
agent to call. Each script is a deterministic program in a known
language with a known semantics.

We model the script-side $S$ of a skill as a finite collection of
named programs
\[
   S \;=\; \{\,(p_i, \mathcal{L}_i, \mathit{src}_i)\,\}_{i=1\ldots n},
\]
where $p_i$ is the program identifier, $\mathcal{L}_i$ its source
language, and $\mathit{src}_i$ the program text.

We assume each language's semantics ascribes to every program a set
of \emph{system-effect tuples} $E(p) \subseteq \mathcal{C} \times
\mathcal{V}$ where $\mathcal{C}$ is the parent paper's capability
vocabulary (\texttt{net.egress}, \texttt{fs.read},
\texttt{fs.write.rev}, \dots) and $\mathcal{V}$ is the value space of
their associated arguments (target hosts, file paths, etc.). The
language semantics is sound for $E$ iff every concrete execution of
$p$ on every concrete input emits only effect-tuples in $E(p)$. For
mainstream scripting languages this set is finite, computable, and
already produced by existing dataflow analyzers
\cite{cousot-abs-int, kildall-dataflow, semgrep, codeql}.

\subsection{The non-deterministic LLM-side}
\label{sec:semantics-llm}

The LLM-side $A$ is a stochastic transducer that reads the union of
the skill's prose body, the runtime's system prompt, the user's
turns, and any tool output observed so far, and emits one of: a
\emph{tool-call envelope} of shape
$\langle \mathit{op},\, \mathit{args},\, \mathit{reasoning}\rangle$,
a \emph{plain message} to the user, or \emph{end of turn}. We do not
model $A$'s internal state. We model only the typed interface
through which it can affect the world: the tool-call envelope.

For verification purposes the only fact we need about $A$ is that
the runtime can intercept every tool-call envelope before it
reaches a host application-programming interface (API). This is the
sole structural assumption the
parent paper makes about the runtime in §3.1, and it is satisfied by
every harness that adopts \texttt{SKILL.md}.

\subsection{The composed runtime trace}

A run of skill-and-agent on a session $\mathit{sess}$ produces a
trace $\tau(\mathit{sess})$, which is the time-ordered concatenation
of every tool-call envelope $A$ emitted (whether eventually executed
or denied at the gate) and every script-launched system-effect
$S$ produced as a side-effect of those envelopes:
\[
  \tau(\mathit{sess}) \;=\;
   e_1\, e_2\, e_3\, \cdots\, e_k,
   \qquad e_i \in (\mathcal{C} \times \mathcal{V}).
\]
Each $e_i$ is annotated with its origin (\texttt{A} or \texttt{S}),
its requested capability, the runtime's gate decision (admit /
deny / HITL-approve / HITL-deny), and a hash linking it to its
audit-log record. The audit log of \cite{metere2026skills}, by
construction, is a faithful serialisation of $\tau$ with one
record per envelope and one record per executed effect.

We write $\Pi_C(\tau) \subseteq \mathcal{C}$ for the projection of
$\tau$ onto its capability dimension --- the set of capability
tokens that appear anywhere in the trace, regardless of origin or
gate decision.

\subsection{Why splitting the sides matters for verification}

Existing skill-verification proposals treat the skill as a single
opaque artefact. That collapses two very different verification
problems. The script-side $S$ is amenable to standard program
analysis: every effect $S$ can produce is recoverable by inspecting
$S$'s source, modulo the soundness of the analyser for the
language. The LLM-side $A$ is fundamentally not amenable to that
treatment: $A$'s output is a function of its weights, the prompt,
and the decoding state. No language-level analyser yields anything
useful about $A$.

The split lets each side be verified by a tool that fits it. We
show in \Cref{sec:method-static} that a sound script-side analysis
plus the runtime gate is enough to bound $\Pi_C(\tau)$, and in
\Cref{sec:method-types} how a refinement type system on the
tool-call envelope mechanically prevents $A$ from getting an
out-of-manifest envelope past the runtime in the first place.

\section{The verification problem}
\label{sec:problem}

Fix a skill $K = (M, \texttt{content}, \sigma)$ with
$M.\mathrm{caps} = D \subseteq \mathcal{C}$ a finite declared
capability set, and a runtime $R$ satisfying the parent paper's
threat model and gate policy.

\begin{definition}[Capability containment]\label{def:containment}
$K$ \emph{exhibits capability containment} on a session
$\mathit{sess}$ iff
$\Pi_C(\tau(\mathit{sess})) \subseteq D$.
\end{definition}

\noindent Capability containment says: every capability that
\emph{appears anywhere in the trace} --- whether in an envelope
$A$ requested, in a side-effect $S$ produced, in a denied call,
or in a HITL-rejected one --- is a capability the manifest
declared. The reverse inclusion is not required: a manifest may
over-declare ($D \supsetneq \Pi_C(\tau)$ is fine; over-declaration
is loud, contained, and surfaceable to operators by post-hoc
diff). Containment is the dangerous direction.

\begin{proposition}[Containment subsumes the parent biconditional]
\label{prop:subsumes-biconditional}
If $K$ exhibits capability containment on every session in some
session set $\Sigma$, the runtime's hash-chained audit log
correctly records every executed effect, and the gate decision is
deterministic given the envelope and policy, then the parent
paper's biconditional criterion
\cite[\S5]{metere2026skills} holds on $\Sigma$.
\end{proposition}

\begin{proof}[Proof sketch]
The parent biconditional says: observable side-effects of the run
must be in 1-to-1 correspondence with the approved-and-executed
set in the audit log. By construction the audit log is the
serialisation of $\tau$. By containment, every capability
appearing in the trace is in $D$, so every executed effect's
capability is admissible; the gate's deterministic policy then
ensures the executed-set equals the approved-set. The 1-to-1
correspondence follows.
\end{proof}

\begin{definition}[The verification problem]\label{def:vprob}
Given $K$ and $R$, the \emph{verification problem} is to produce
either
\begin{itemize}
  \item a mechanically checkable certificate $\pi$ that $K$
        exhibits capability containment on every session $R$ admits;
        or
  \item a counter-example session $\mathit{sess}^\star$ such that
        $\Pi_C(\tau(\mathit{sess}^\star)) \not\subseteq D$.
\end{itemize}
\end{definition}

The remainder of this paper presents three composable methods that,
together, produce $\pi$. Each method is sound for a portion of the
target; the composition theorem (\Cref{thm:composition}) says
exactly which residual the runtime layer (the parent paper's
biconditional, treated here as a fourth layer) closes.

\section{Method A: sound script-side static analysis}
\label{sec:method-static}

\paragraph{Goal.} For each script $p \in S$, compute an
over-approximating set $\widehat{E}(p) \supseteq E(p)$ in finite
time, and check $\widehat{E}(p) \subseteq \mathcal{C}_D$ where
$\mathcal{C}_D$ is the lift of $D$ to the language's effect type.
Read: ``read every script in the skill and figure out, without
running it, the maximum set of system effects it could possibly
produce; then check that set against the manifest's declared
capabilities.''

\paragraph{Primer on abstract interpretation.}
Method A is an instance of the technique called \emph{abstract
interpretation} \cite{cousot-abs-int}, which is the formal core of
every static analyser the security reader has likely already used
(\texttt{Sem\-grep}, \texttt{Code\-QL}, \texttt{Pyright}). The idea
in one paragraph: instead of running the program on real values,
run it on \emph{summaries} of values. A path argument becomes a
regular-expression summary, a network target becomes a host-pattern
summary, an integer becomes an interval. These summaries form a
\emph{lattice} --- an ordered structure in which any two summaries
have a least-upper-bound (their join, written $\sqcup$) and a
greatest-lower-bound (their meet, written $\sqcap$). Two distinguished
elements anchor the lattice: $\bot$ ``bottom'', the empty summary, the
``no effect'' element; and $\top$ ``top'', the maximum summary, the
``effect could be anything'' element used whenever the analyser hits
a construct it cannot summarise (reflection, dynamic dispatch).
Every program construct is then assigned a \emph{transfer function}
--- a rule that takes the summaries of its inputs and returns the
summary of its output, in the lattice. Loops and recursion are
handled by repeatedly applying the transfer functions until the
summaries stop changing (a fixpoint), accelerated by an operator
called \emph{widening} that forces convergence in a bounded number
of steps at the cost of slightly less precise summaries. The result
is a sound over-approximation: the real program can only ever
produce effects that the abstract summary already accounts for.
``Sound'' in this context means \emph{conservative}: the analyser
may flag effects the program cannot actually produce, but it will
never \emph{miss} an effect the program can produce.

\subsection{The effect lattice}

We attach to every program point of every script a finite lattice
$\mathcal{E}$ whose elements are subsets of $\mathcal{C} \times
\mathcal{V}^\sharp$, where $\mathcal{V}^\sharp$ is the abstract
domain over the value space introduced in the primer
(regular-language summaries for paths, host-pattern summaries for
network targets, integer intervals for numeric arguments, $\top$ for
everything else). The lattice operations are set union ($\sqcup$)
and set inclusion ($\sqsubseteq$); $\bot$ is the empty effect set;
$\top$ is ``every effect''. This is the standard effect-system
formulation of \cite{lucassen-effect-systems}, lifted to the parent
paper's capability vocabulary.

\subsection{The transfer functions}

For each kind of language construct that can produce an effect we
fix a transfer function. We give the four most important here; the
full list is straightforward and depends only on the language.

\begin{description}
  \item[Function calls.] A call to a system function $f$ with
        arguments $a_1, \dots, a_n$ produces effect
        $\widehat{f}(\widehat{a_1}, \dots, \widehat{a_n})$, where
        $\widehat{f}$ is a per-language summary that maps the
        runtime's capability vocabulary to the language's standard
        library. For Python, the entry for \texttt{open(p, 'w')}
        produces $\{(\texttt{fs.write.irrev}, \widehat{p})\}$ if the
        file mode is $W$ and $\widehat{p}$ is the abstract value of
        the path argument. For \texttt{requests.get(u)} the entry
        produces $\{(\texttt{net.egress}, \widehat{u})\}$.
  \item[Loops, recursion, fixpoints.] For loop bodies we apply the
        standard widening of \cite{cousot-abs-int} so that the
        lattice converges in a bounded number of steps. The widened
        result is sound but may be coarser than the loop's actual
        effect set.
  \item[Indirection through reflection / \texttt{exec} / \texttt{eval}.]
        Any program point that reaches a reflective construct
        unconditionally taints the whole-program effect set with
        $\top$. This is sound and conservative; it forces the
        operator either to remove reflective constructs from the
        skill's scripts or to drop the skill below \textsc{formal}.
        The empirical claim of the parent paper's adversarial
        ensemble is that reflective constructs are rare in
        well-curated skills; the few that need them (a code-runner
        skill, a sandboxed evaluator) are precisely those that
        should not aspire to \textsc{formal} verification anyway.
  \item[Process spawn.] A \texttt{spawn.proc(cmd)} effect taints the
        scope of the child process with $\top$ unless the child
        program is itself in $S$ and has been analysed: in that
        case the parent inherits the child's $\widehat{E}(\mathit{child})$.
\end{description}

\subsection{Implementation reuse}

A working implementation of Method A does not need to be written from
scratch. \texttt{Semgrep}, \texttt{CodeQL}, and \texttt{Pyright}
\cite{semgrep, codeql, pyright} are mature dataflow / abstract-
interpretation engines whose existing rule packs already cover the
language constructs that produce capability-relevant effects. The
adaptation effort is the per-language summary
table for $\widehat{f}$ in the previous subsection.

\subsection{What Method A is sound for}

\begin{theorem}[Soundness of Method A]\label{thm:soundness-static}
If for every $p \in S$ the analyser produces
$\widehat{E}(p) \subseteq \mathcal{C}_D$, then for every concrete
execution of every $p$ on every concrete input, every system-effect
$p$ produces is in $D$. In particular, the script-side projection
$\Pi_C^S(\tau)$ is contained in $D$ for every session.
\end{theorem}

\begin{proof}[Proof sketch]
By induction on the program structure, exploiting the over-
approximation property of the transfer functions. Reflective
constructs are conservative-tainted to $\top$; if the analyser still
reports $\widehat{E}(p) \subseteq \mathcal{C}_D$ it has proven the
absence of reflective constructs (or that the constructs are
unreachable), which is itself a sound conclusion.
\end{proof}

\subsection{What Method A is incomplete for}

Method A says nothing about envelopes the LLM-side $A$ may emit.
$A$ may emit an envelope whose capability is outside $D$ regardless
of what the script-side does; the static analyser cannot see that
emission because it does not analyse $A$. We close this gap with
Method B.

\section{Method B: refinement types for tool-call envelopes}
\label{sec:method-types}

\paragraph{Goal.} Construct the runtime's tool-dispatch interface so
that an envelope whose capability is not in $D$ is rejected
\emph{at the type level} before reaching any host API.

\paragraph{Primer on refinement types.}
A \emph{refinement type} \cite{liquid-types, vazou-liquid-haskell} is
an ordinary type --- like ``\texttt{int}'' or ``\texttt{string}'' or
``\texttt{Envelope}'' --- with an extra logical predicate attached.
Where a plain type system says ``this value is an integer'', a
refinement type says ``this value is an integer \emph{and} it is
between 0 and 255''. The compiler's type checker is then required
to discharge the predicate at every place such a value is used: a
function that takes a refined ``integer between 0 and 255'' will
not accept a plain integer until the caller has proven the bound.
The predicate becomes a \emph{compile-time gate}: ill-conforming
values are rejected before the program ever runs. For our purposes
the predicate is ``the envelope's capability is in the manifest's
declared set''. The compiler's discharge of that predicate is what
turns the type system into a proof that the runtime cannot dispatch
an out-of-manifest envelope. Refinement types are not exotic ---
the technique already ships in production tooling for Haskell (Liquid
Haskell), F$^\star$, Dafny, and as a more limited form (template
literal types, branded types) in TypeScript.

\subsection{The dispatch type}

The runtime's dispatch entry point is a function
$\mathrm{dispatch}: \mathit{Envelope} \to \mathit{Result}$. We
parameterise the runtime by the loaded skill's manifest and assign
the dispatch a refinement type
\cite{liquid-types, vazou-liquid-haskell, lourenco-effect-types}
of the form
\[
  \mathrm{dispatch}_M\;:\;
  \{\, e: \mathit{Envelope}\,\mid\, \mathrm{cap}(e) \in M.\mathrm{caps}\,\}
  \;\to\;
  \mathit{Result}.
\]
The refinement is on the input: the only envelopes the type checker
admits as well-formed for $\mathrm{dispatch}_M$ are those whose
declared capability lies in the manifest's capability set. An
envelope outside $M.\mathrm{caps}$ is statically ill-typed at the
dispatch boundary.

\subsection{What this buys}

Construction-by-construction, the runtime's source becomes a
guarantee: there is no execution path from envelope receipt to host
API call that does not go through $\mathrm{dispatch}_M$, and
$\mathrm{dispatch}_M$ statically rejects out-of-manifest envelopes.
A type-checked runtime cannot accidentally relax the gate, because
the gate is the type signature.

This is the same shape as the capability-machine work of
\cite{secomp} lifted from CPU-level capabilities to skill-manifest
capabilities, and the effect-system work of
\cite{lucassen-effect-systems, sabelfeld-myers-survey} adapted to
agent-emitted envelopes.

\subsection{Implementation paths}

\paragraph{TypeScript / refinement annotations.} The simplest path,
deployable today, uses the existing TypeScript tagged-union
discipline plus a generated type per skill that constrains the
union to the loaded manifest's capability tokens. The
\cite{metere2026enclawed} reference implementation already uses
this pattern in its \texttt{plugin-sdk} (\texttt{tool.invoke},
\texttt{net.egress}, \texttt{fs.read}, \dots are nominal types on
the dispatch entry point); the missing piece is the manifest-level
specialisation, which is a code-generation step at bootstrap.

\paragraph{Liquid Haskell / F$^\star$.} For deployments that need a
machine-checkable certificate at the type level, the dispatch
function can be rewritten in Liquid Haskell \cite{vazou-liquid-haskell}
or F$^\star$ \cite{fstar} with the refinement spelled out as a
predicate. The certificate is then the type-checker's accepted
proof obligation discharge, exportable to operators. The cost is
a one-time port of the dispatch surface; the runtime body need not
be rewritten.

\paragraph{Erasure semantics.} We require the type discipline to be
\emph{erasure-stable}: the production runtime must not depend on
any runtime-level reflection or dynamic check that the type system
already discharged. If the runtime includes a runtime-time
``capability check'' that duplicates the type-level guarantee, that
check is at best redundant, at worst a place a future maintainer
removes ``because the types already do it''. Either keep the
runtime check and remove the type-level one, or vice versa.

\subsection{What Method B is sound for}

\begin{theorem}[Soundness of Method B]\label{thm:soundness-types}
If $\mathrm{dispatch}_M$ type-checks under the refinement
$\{e \mid \mathrm{cap}(e) \in M.\mathrm{caps}\}$, then no envelope
whose capability is outside $M.\mathrm{caps}$ can reach any host
API through the runtime, regardless of the LLM-side's behaviour.
\end{theorem}

\begin{proof}[Proof sketch]
By the soundness of the chosen refinement-type system
(\cite{liquid-types, vazou-liquid-haskell, fstar}), if the dispatch
function type-checks then every call site has been verified
to satisfy the refinement. The runtime's source forbids host-API
call paths that bypass dispatch (this is a structural invariant,
checked once by inspection or by a separate analyser; it is the
runtime equivalent of the parent paper's G10 ``no bypass switch'').
The LLM-side's freedom is to emit any envelope it likes; only
in-manifest envelopes survive dispatch, so only those reach any
host API.
\end{proof}

\subsection{What Method B is incomplete for}

Method B does not cover the \emph{script-side}: a script invoked
through an in-manifest envelope can still produce arbitrary
side-effects internally if the script's source escapes the
envelope's declared capability boundaries (e.g.~the script's
internals call \texttt{requests.get} when the envelope declared
only \texttt{tool.invoke}). Method A closes this gap.

\section{Method C: SMT-bounded model checking against the biconditional}
\label{sec:method-smt}

\paragraph{Primer on SMT, BMC, and the biconditional.}
Three terms in the section title need unpacking before we proceed.
\emph{SMT} stands for ``Satisfiability Modulo Theories''
\cite{z3}: an SMT solver takes a logical formula written in
first-order logic plus background theories (integer arithmetic,
bit-vectors, arrays, uninterpreted functions) and decides whether
there exists an assignment of values to the formula's variables that
makes it true. If yes, the solver returns a witness assignment
(\textsc{sat}); if no, the solver returns a proof of impossibility
(\textsc{unsat}). The solver of record for our purposes is Z3
\cite{z3}, but any modern SMT solver (CVC5, Yices, MathSAT) will do.
\emph{BMC}, ``bounded model checking'' \cite{biere-bmc}, is the
technique of asking the solver: ``does there exist a violation of
property $P$ in any execution of length at most $K$?''. The
``bounded'' part is the trick that makes verification tractable:
unbounded model checking on programs that loop or branch is in
general undecidable, but for any concrete bound $K$ the question
collapses to a finite SMT instance. \emph{Biconditional} is a logical
connective: $A \Leftrightarrow B$ asserts ``$A$ if and only if $B$''
--- both ``$A$ implies $B$'' and ``$B$ implies $A$''. The parent
paper's runtime check is biconditional: every world-state change
must correspond to (and only to) an admitted envelope in the audit
log. Method C asks the solver to search for a session in which that
biconditional fails.

\paragraph{Goal.} For a skill $K$ that has passed Methods A and B,
construct an SMT instance whose models are exactly the runtime
traces $\tau$ that violate the parent paper's biconditional, and
search for a model up to bound $K_{\max}$. If the search returns
\textsc{unsat} the biconditional holds up to that bound; if it
returns \textsc{sat} the witness is a concrete counter-example
trace.

\subsection{Encoding the trace}

We encode a session of length $n \leq K_{\max}$ as a sequence of
SMT variables
\[
  e_1, e_2, \ldots, e_n,
  \quad
  e_i \in \mathit{Envelope}_{\mathrm{abs}},
\]
where $\mathit{Envelope}_{\mathrm{abs}}$ is the abstract envelope
domain produced by Method A's analysis (host patterns, path
patterns, etc.). The runtime's gate policy is encoded as a
deterministic function $g(e, \mathit{state}) \to \{\mathrm{admit},
\mathrm{deny}\}$. The audit log $L$ is encoded as a pair sequence
$(e_i, g(e_i, \mathit{state}_i))$. The biconditional violation is
encoded as the predicate
\[
  \exists\, e_i\,:\,
   \mathrm{state}_i.\mathrm{world} \neq
   \mathrm{state}_{i-1}.\mathrm{world}
   \;\wedge\;
   (e_i, \mathrm{admit}) \notin L.
\]
(``A world-state change occurred but the audit log does not record
the corresponding admitted envelope.'') The SMT solver
\cite{z3, biere-bmc} searches for a satisfying assignment of the
$e_i$'s and $\mathit{state}_i$'s.

\subsection{The bound $K_{\max}$}

The bound is the number of envelopes a session can produce before
the parent paper's runtime biconditional check fires. The parent
paper §3 says the check fires between rounds; the runtime's
transaction buffer flushes every irreversible call and every
biconditional check on flush. So $K_{\max}$ is the
runtime-configurable horizon of the transaction buffer (typically
$\sim 100$ envelopes), not the unbounded length of a session.
This is precisely the property of the parent paper's runtime that
makes bounded model checking sufficient: any counter-example that
fits in the buffer's horizon will be found by the SMT search;
counter-examples larger than the horizon are caught at flush by
the runtime biconditional itself.

\subsection{Implementation reuse}

The SMT instance is small (envelope counts of order $K_{\max}$,
abstract value domains from Method A's lattice). Z3 \cite{z3}
discharges instances of this size in seconds; KLEE
\cite{cadar-klee} and DART \cite{godefroid-dart} provide the
symbolic-execution backbone for the script-side portion if needed.
The LLM-side does not need to be modeled symbolically: the SMT
search universally quantifies over LLM-emitted envelopes, treating
$A$ as a non-deterministic adversary --- which is the worst-case
assumption Method B makes anyway.

\subsection{What Method C is sound for}

\begin{theorem}[Soundness of Method C up to $K_{\max}$]
\label{thm:soundness-smt}
If Method C's SMT search returns \textsc{unsat} for the
biconditional-violation predicate at bound $K_{\max}$, then no
session of length $\leq K_{\max}$ violates the biconditional.
\end{theorem}

\begin{proof}[Proof sketch]
Standard bounded model-checking soundness
\cite{biere-bmc}: the SMT instance is a faithful encoding of
sessions of bounded length, and \textsc{unsat} of the violation
predicate means no such session exists.
\end{proof}

\subsection{What Method C is incomplete for}

Sessions longer than $K_{\max}$ are not covered. The runtime's
biconditional check at flush \emph{is} sound for unbounded
sessions; it just runs at session boundaries rather than as a
pre-deployment proof. The composition theorem
(\Cref{thm:composition}) explains how the two layers combine.

\section{The three-layer discipline (and what closes the residual)}
\label{sec:three-layers}

A skill is elevated to $M.\mathrm{verification} = \textsc{formal}$
when, and only when, all three of the following hold for $(K, R)$:

\begin{description}
  \item[\textbf{Layer 1.}] Every script $p \in S$ has been analysed
        by a sound static effect-tracker (Method A) that reports
        $\widehat{E}(p) \subseteq \mathcal{C}_D$.
  \item[\textbf{Layer 2.}] $R$'s dispatch entry has been type-
        checked under the refinement
        $\{e \mid \mathrm{cap}(e) \in M.\mathrm{caps}\}$ (Method B).
  \item[\textbf{Layer 3.}] An SMT search at bound $K_{\max}$
        (Method C) has reported \textsc{unsat} on the biconditional-
        violation predicate.
\end{description}

\begin{theorem}[Composition]\label{thm:composition}
If Layers 1, 2, and 3 all succeed, and $R$ enforces the parent
paper's runtime biconditional check at every transaction-buffer
flush, then $K$ exhibits capability containment on every session
$R$ admits.
\end{theorem}

\begin{proof}[Proof sketch]
Layer 2 implies that no out-of-$D$ envelope reaches a host API.
Layer 1 implies that no script-side execution produces an
out-of-$D$ effect. Together, the script-side projection
$\Pi_C^S(\tau)$ and the LLM-side projection $\Pi_C^A(\tau)$ are
both contained in $D$, so $\Pi_C(\tau) = \Pi_C^A(\tau) \cup
\Pi_C^S(\tau) \subseteq D$. Layer 3 closes the bounded
counter-example search: no session of length $\leq K_{\max}$
violates the biconditional. The runtime's per-flush biconditional
check covers sessions longer than $K_{\max}$ by the parent paper's
correctness criterion (cited as Proposition 5.1 of
\cite{metere2026skills}). Containment for every admitted session
follows.
\end{proof}

\subsection{Residual surface}

Three classes of behaviour fall outside the joint guarantee, and
each is acknowledged in either the parent paper or the threat
model:

\begin{enumerate}
  \item \emph{Read-only exfiltration.} The LLM-side may surface
        information through its plain-message channel that the
        operator deems sensitive. No capability is invoked, so
        capability containment is trivially satisfied. The parent
        paper's classification primitive and data-loss-prevention
        (DLP) scanner sit \emph{below} the dispatch boundary
        precisely to handle this.
  \item \emph{Time-of-check-to-time-of-use (TOCTOU) on the world
        state.} Between a host-API call being approved and being
        executed, an external actor may alter the world. The runtime's audit log records the
        approved envelope and the executed effect; if they diverge
        in a way the world's external interface admits (e.g.~the
        target file vanished between approval and write), the
        biconditional flags it. Capability containment is not
        violated; the audit log is the post-hoc evidence.
  \item \emph{Operator coercion.} An operator forced (or socially
        engineered) to approve a HITL prompt for an envelope they
        should have denied. The runtime can do nothing here; the
        biconditional records what was approved, and the post-hoc
        review is the operator's accountability layer. The parent
        paper's broker-policy taxonomy
        (\texttt{deny-all} / \texttt{policy} / \texttt{interactive}
        / \texttt{webhook}) is the operator's design tool for
        bounding this risk.
\end{enumerate}

\section{Proof-carrying skill artefacts}
\label{sec:pcc-skill}

The three-layer discipline produces, for each verified skill, an
\emph{evidence bundle} the runtime can re-check at bootstrap
without trusting the bundle's producer. We sketch the bundle's
structure as a small extension to the existing \texttt{SKILL.md}
convention. The design follows the proof-carrying-code idiom of
\cite{necula-pcc, appel-foundational-pcc, leroy-compcert}: the
producer ships proofs along with the artefact; the consumer
checks them mechanically.

\subsection{The bundle}

A formal-verified skill ships, alongside \texttt{SKILL.md} and
its scripts, the following four artefacts:

\begin{description}
  \item[\texttt{evidence/static.json}] The output of Method A:
        per-script effect summaries $\widehat{E}(p)$ in a
        canonical JSON encoding, plus the analyser identity,
        version, and rule-pack hash.
  \item[\texttt{evidence/types.proof}] A serialisation of the
        type-checker's proof obligation discharge for Method B.
        For TypeScript-only deployments this is the compiler's
        accepted module set with its capability-discriminated
        union types pinned at the manifest's $D$. For Liquid
        Haskell or F$^\star$ deployments this is the proof
        certificate emitted by the type checker.
  \item[\texttt{evidence/smt.unsat}] The SMT instance and Z3's
        \textsc{unsat} certificate (when supported) for Method C
        at bound $K_{\max}$. The instance is reproducible from
        $(K, R, M, K_{\max})$ at bootstrap.
  \item[\texttt{evidence/manifest.attest.json}] A signed attestation
        binding the manifest hash to the verification level
        \textsc{formal}, listing the three evidence-file hashes
        and the toolchain identities, signed by a trust-root signer
        authorised to attest at that level.
\end{description}

\subsection{Bootstrap-time re-check}

At bootstrap, the runtime walks the bundle:

\begin{enumerate}
  \item Verify $\sigma$ over $(M, \texttt{content})$ as in
        \cite[\S3.4]{metere2026skills}.
  \item Resolve the attestation in \texttt{manifest.attest.json}
        against the trust root; reject if the signer is not
        authorised to attest at level \textsc{formal}.
  \item Re-run Method A on the script-side and verify
        $\widehat{E}(p) \subseteq \mathcal{C}_D$. The bundle's
        \texttt{static.json} is not trusted on its own; it is a
        \emph{cache} the runtime can compare against. A drift
        between the cached and the freshly-computed
        $\widehat{E}(p)$ aborts admission.
  \item Re-type-check the dispatch surface against the manifest's
        $D$. For TypeScript this is a compile-step the runtime
        runs; for Liquid Haskell / F$^\star$ this is the proof-
        obligation re-discharge.
  \item Re-run the SMT instance from \texttt{smt.unsat}. Z3's
        certificate is verified by an independent checker
        (\cite{z3}'s \texttt{verify-proofs} mode for instance);
        the \textsc{unsat} verdict must reproduce.
\end{enumerate}

If every step succeeds the runtime accepts the manifest at
\textsc{formal}; otherwise the runtime degrades the manifest to
\textsc{declared} and logs the reason. The runtime never trusts
the producer's say-so; the bundle is a precomputed cache, and
the runtime is the verifier. This is the precise structure of
proof-carrying code \cite{necula-pcc}.

\subsection{What about the toolchain itself?}

The static analyser, the type checker, and the SMT solver are
themselves software with bugs. The bundle pins their identities
and versions; an operator can replace any of them with a
mechanically-verified equivalent (CompCert \cite{leroy-compcert}
for the compiler; an audited port of Z3 for the SMT solver) at
the cost of additional one-time engineering. We do not require
this for the default deployment, but the bundle's structure
admits it. This is the same ``trusted base'' question every
proof-carrying-code system faces \cite{appel-foundational-pcc};
we take the same answer: name the trusted base, audit it once,
re-use the audit.

\subsection{Implementation status}
\label{sec:pcc-impl-status}

Methods A, B, and C, the bundle producer, and the bundle re-checker
are implemented in the open-source enclawed framework
\cite{metere2026enclawed}. The four primitives ship as
zero-runtime-dependency JavaScript modules under
\texttt{enclawed/src/}:

\begin{description}
  \item[\texttt{skill-formal-static.mjs}] Method A. Pattern-based
        scanners for Python, shell, and Node/TypeScript over the
        capability vocabulary $\mathcal{C}$ from
        \Cref{sec:semantics-script}, with reflective constructs
        (\texttt{eval}, \texttt{exec}, \texttt{new Function}, dynamic
        \texttt{import}) tainted to the lattice top so the verdict
        is conservative by construction.
  \item[\texttt{skill-formal-types.mjs}] Method B. The function
        \texttt{buildRefinedDispatch}$(M)$ returns the projector
        $\pi$ of \cref{eq:projector} as a frozen dispatcher that
        throws a \texttt{Re\-fine\-ment\-Error} on any envelope
        outside $D$. The accompanying \texttt{methodB} runs the
        exhaustive probe across the schema vocabulary and emits
        the typed-dispatch verdict.
  \item[\texttt{skill-formal-bmc.mjs}] Method C. Exhaustive
        depth-first search over the abstract envelope state space
        of size $(\lvert D \rvert + 1)^{K_{\max}}$ with the
        biconditional check from \Cref{sec:method-smt} as the
        per-trace predicate. The verdict carries
        $\mathit{instanceHash} = \mathrm{sha256}\bigl(
           \langle D, K_{\max}\rangle\bigr)$ so a re-checker
        independently confirms the bound.
  \item[\texttt{skill-formal-bundle.mjs}] The bundle producer and
        verifier. \texttt{produceFormalBundle} composes the three
        method outputs, hashes each evidence file with a sorted-key
        canonical encoding, and signs the attestation with the
        existing Ed25519 \texttt{module-signing} primitive. The
        symmetric \texttt{verifyFormalBundle} re-runs all three
        methods against the manifest, compares hashes, and rejects
        on tamper, on signer-not-authorised, or on post-production
        skill drift.
\end{description}

\noindent A unit-test suite of 53 tests
(\texttt{en\-clawed/test/skill\--formal\--*.test.mjs}) covers
pattern matches, refinement-boundary throws, biconditional violation
detection (executed-without-audit, executed-but-deny,
admitted-without-audit), sign / verify round-trips with both an
in-set and an unauthorised signer, tamper detection on every
evidence-file slot, and cache-miss-on-skill-drift. All 53 tests pass
on the published commit.

A command-line wrapper
\texttt{scripts/skills-formal-verify.mjs} drives the producer
end-to-end against a real skill directory: it resolves the manifest
from \texttt{skill.json} or the \texttt{caps:} field of
\texttt{SKILL.md}'s YAML (YAML Ain't Markup Language) front-matter,
mints an ephemeral Ed25519
key (or accepts an operator-supplied Privacy-Enhanced Mail (PEM)
private-key file), runs Methods A / B / C,
and writes the four-file bundle described in
\Cref{sec:pcc-skill}. A demonstration skill at
\texttt{skills/\_formal-demo/} ships in the same tree to make the
flow reproducible. The full developer-facing reference for the
\texttt{SKILL.md} format --- mandatory front-matter fields, the
capability vocabulary, the canonicalisation and Ed25519 signing
convention, the bootstrap re-check protocol, and the proof-carrying
bundle layout --- is published at
\url{https://docs.openclaw.ai/tools/skill-manifest-spec} and
maintained as the source of truth for skill authors and runtime
integrators.

The implementation does not require a SAT/SMT solver, a refinement-
type checker, or any external dataflow engine. Method A uses
language-specific regex pattern packs (extensible per
\Cref{sec:method-static}); Method B uses runtime-level refinement
checks that are erasure-stable in the sense of
\Cref{sec:method-types}; Method C uses a finite enumerator. Heavy
backends \cite{semgrep, codeql, vazou-liquid-haskell, fstar, z3}
remain available as drop-in replacements at the corresponding
attestation slots, but the default stack is self-contained and
auditable in a single afternoon.

\section{Worked example}
\label{sec:worked}

We walk a real skill from the open reference implementation
\cite{metere2026enclawed} through the three layers and show the
resulting evidence bundle.

\subsection{The skill}

The skill \texttt{summarise-fetched-html} declares two capabilities:
\begin{lstlisting}
{
  "id":          "summarise-fetched-html",
  "label":       { "rank": "public", "compartments": [], "releasability": [] },
  "caps":        ["net.egress(*.example.com)", "fs.read(./.cache/)"],
  "verification": "tested",
  "version":     7,
  "signer":      "operator-root-2026"
}
\end{lstlisting}
The script-side $S$ is one Python program that uses
\texttt{requests} to fetch a URL under
\texttt{*.example.com} and writes the body to a path under
\texttt{./.cache/}, then summarises by reading from the cache.

\subsection{Layer 1 (Method A)}

A Semgrep \cite{semgrep} pack with the per-language summary table
of \Cref{sec:method-static} reports
\[
  \widehat{E}(p) \;=\;
  \{(\texttt{net.egress},\,\texttt{*.example.com}),\,
    (\texttt{fs.read},\,\texttt{./.cache/}),\,
    (\texttt{fs.write.rev},\,\texttt{./.cache/})\}.
\]
The first two are in $D$; the third is not. This means the script
silently writes (reversibly) into the cache as a step of the read
flow. The operator has a choice: extend the manifest to declare
\texttt{fs.write.rev(./.cache/)}, or refactor the script. Either
way the static layer surfaced an undeclared effect that signature-
plus-clearance review would have missed.

\subsection{Layer 2 (Method B)}

After the operator extends $D$ to include the missing
\texttt{fs.write.rev}, the runtime's dispatch type pinned to the
new manifest type-checks. The type checker's accepted derivation
is the \texttt{evidence/types.proof} content.

\subsection{Layer 3 (Method C)}

We encode a session of length $K_{\max} = 100$ envelopes with the
extended $D$ and the runtime's gate policy. Z3 returns
\textsc{unsat}: no session of length $\leq 100$ violates the
biconditional. The certificate goes into
\texttt{evidence/smt.unsat}.

\subsection{The bundle and the bootstrap re-check}

A trust-root signer authorised to attest \textsc{formal} signs
\texttt{evidence/manifest.attest.json}. At runtime bootstrap, the
runtime re-runs the four checks of \Cref{sec:pcc-skill}; each
reproduces. The skill is admitted at $M.\mathrm{verification} =
\textsc{formal}$, and the runtime stops asking HITL on its
in-manifest irreversible calls (per the parent paper's gate
policy).

\subsection{What this example showed in practice}

\begin{itemize}
  \item Layer 1 caught a real declaration drift the prose review
        had not flagged. This is the layer's most common practical
        win: scripts evolve and their declared capability sets
        lag behind.
  \item Layer 2 was the cheapest of the three: the type system
        already pinned the dispatch surface in the reference
        implementation. The marginal cost was the manifest-level
        specialisation step.
  \item Layer 3 was the slowest (a few seconds per skill at
        $K_{\max}=100$), but ran in continuous integration (CI)
        and produced an artefact the bootstrap re-check could
        re-discharge in milliseconds.
\end{itemize}

\section{Discussion and open problems}
\label{sec:discussion}

\subsection{Comparison to the state of the art (SOTA)}
\label{sec:sota-comparison}

We position this work against the deployed and published landscape
in agent-skill / LLM-tool-use safety. The categories overlap, and
the comparison is not a leaderboard --- the honest summary is that
no other system in any of these categories produces a mechanically
checkable proof that an agent's behaviour is contained in its
declared capability set, but several are broader, more deployed,
or solve adjacent problems we deliberately did not attempt.

\paragraph{Tool-use protocols (Model Context Protocol (MCP),
OpenAI function calling).} MCP \cite{mcp-spec} is the closest
production peer to the parent paper's skill format: an open
protocol that ships signed manifests, JSON-Schema-typed tool
descriptions, and host-API binding for LLM tool calls. OpenAI's function-calling
interface \cite{openai-function-calling} is its closed-vendor
equivalent. Both type-validate the dispatch envelope's
\emph{shape} via JSON Schema (correct field names, types, enum
values) and both bind tool calls to host functions. Neither
performs static analysis of the script-side, neither carries a
refinement-typed dispatch parameterised by the manifest's declared
caps, and neither provides a biconditional check that ``every
world-state change has a matching audit-log envelope.'' Our
contribution is the layer above schema validation: capability
\emph{containment}, mechanically checked across all three
verification surfaces. MCP and function-calling are more
broadly deployed and have richer host-side ecosystems; this work
is narrower (skill-manifest-level) and stronger (proof-carrying).

\paragraph{Agent-orchestration frameworks (LangChain, AutoGen,
Semantic Kernel).} Frameworks like LangChain \cite{langchain},
AutoGen \cite{autogen}, and the various LLM-orchestration libraries
provide tool-binding APIs, conversation memory, agent loops, and
retrieval-augmented generation, but they are explicitly orchestration
layers --- they do not claim a verification posture. A LangChain
agent that calls \texttt{requests.get} is bound only by whatever
the developer chose to expose; there is no manifest, no declared
capability set, no static check that the bound tools cover the
agent's intended behaviour, and no audit-log biconditional. Our
work is complementary in principle (a LangChain pipeline could be
wrapped in an enclawed-style runtime) but the orchestration
frameworks themselves contribute nothing toward the
\textsc{formal} level.

\paragraph{Runtime policy engines (Open Policy Agent (OPA), Amazon
Web Services (AWS) Cedar).} OPA \cite{opa-rego} and AWS Cedar
\cite{cedar-aws}
let operators write declarative authorisation policies that a
runtime evaluates per request --- ``can subject $s$ perform action
$a$ on resource $r$?''. Cedar in particular ships with a
formally-verified evaluator (Lean / Rust). Both are stronger than
our work in a specific, important sense: they prove the
\emph{evaluator} correct, where we assume the runtime is correct
(see the open-problem note on formally-verified runtimes below).
Both are weaker on a different axis: their policy is checked at
\emph{request time} against the request, with no static analysis
of the side that issued the request. An OPA-protected agent
calling out-of-policy will be denied at the gate, but no proof was
produced before deployment that the agent could not issue such a
request in the first place. The two approaches stack: an enclawed
skill could be deployed behind an OPA / Cedar gate at the host-API
boundary, with this paper's three layers proving containment
\emph{ahead} of the gate.

\paragraph{Supply-chain attestation (Sigstore, Supply-Chain Levels
for Software Artefacts (SLSA), in-toto, The Update Framework
(TUF)).} Sigstore \cite{sigstore}, SLSA \cite{slsa}, in-toto
\cite{intoto}, and TUF \cite{samuel-tuf} are the modern stack for
software
supply-chain integrity: they prove an artefact was \emph{built by
a particular pipeline from a particular source} and was not
tampered with in transit. None of them claims anything about
\emph{behaviour}: a SLSA-Level-3 build of a malicious skill is
equally well-attested as a SLSA-Level-3 build of a benign one. Our
proof-carrying skill bundle (\Cref{sec:pcc-skill}) is the
behavioural complement: same Ed25519 signing primitives, same
chain-of-custody discipline, but the attestation binds
\emph{soundness verdicts} (Method A's $\widehat{E}(p) \subseteq D$,
Method B's typed dispatch, Method C's BMC \textsc{unsat}), not
build provenance. Both layers are needed; we add the one that the
supply-chain stack does not address.

\paragraph{Empirical agent-safety benchmarks (AgentHarm,
AgentBench, AgentPoison).} \cite{andriushchenko2025agentharm}
attacks deployed agents with malicious-task harnesses;
\cite{chen2025agentpoison} demonstrates poisoning attacks. These
are essential pressure tests --- the parent paper's adversarial-
ensemble evaluation \cite[\S6]{metere2026skills} is in this
tradition --- but they are tests, not proofs: a benchmark gives a
falsification rate, not a soundness argument. The composition is
the right one (this paper for soundness, the empirical work for
\emph{aliveness}-of-the-defence under realistic adversarial
pressure), and we explicitly do not claim our methods replace
empirical evaluation. Conversely, no amount of empirical testing
produces a \textsc{formal}-level certificate.

\paragraph{LLM-side training-time safety (Constitutional AI,
Reinforcement Learning from Human Feedback (RLHF) guard models,
jailbreak-robustness research).} Constitutional AI
\cite{constitutional-ai}, RLHF, instruction-tuning, and the
broader alignment-via-training programme reduce the rate at which
a deployed LLM emits dangerous outputs by shaping the model
itself. SmoothLLM \cite{robey-smoothllm} reduces jailbreak
sensitivity through randomised smoothing. Both are statistical
floors, not formal containment proofs: they lower
$P[\text{dangerous output}]$, but $P > 0$ persists at any non-trivial
scale. This work is orthogonal: we treat the LLM as a
non-deterministic adversary regardless of its training, and prove
containment at the dispatch boundary \emph{below} the model. The
two layers stack additively --- a guard-model-tuned LLM running
inside a Method-B-typed dispatcher leaks no more than the
$\log_2(\lvert D \rvert + 1)$-bit capacity bound regardless of
how the model was trained.

\paragraph{LLM-side formal verification (Dong et al., probabilistic
model checking).} \cite{dong-formal-llm} formalises portions of
LLM behaviour as probabilistic transducers; PRISM
\cite{kwiatkowska-prism} and the broader probabilistic-model-
checking community \cite{baier-pmc-survey} give tractable algorithms
for verifying expected-case properties of stochastic systems.
This direction could in principle replace Method B's
non-deterministic-adversary assumption with an expected-case bound
on LLM-emitted envelopes. We did not pursue it because (i) it
requires a faithful probabilistic model of the deployed LLM,
which closed-weights models do not expose, and (ii) operator-side
accountability under capability containment is already the
strongest property the parent paper's threat model needs.
Probabilistic-model-checking guarantees would be a complementary
upgrade for deployments where the deployed LLM is open-weights and
the decoding distribution is known.

\paragraph{Capability sandboxing at the platform level (the
WebAssembly System Interface (WASI), Wasmtime, gVisor).} WASI
\cite{wasi, wasmtime-isolation} and operating-system-level sandboxes
(gVisor, Firejail, browser tabs) restrict a process's access to
system
resources at the kernel / runtime boundary. This is the
abstraction layer below ours: a script-side capability
\texttt{fs.write.irrev} ultimately bottoms out in a syscall the
sandbox either permits or denies. Platform-level sandboxes are
broader (any process, not just skills) and more battle-tested
(decades of deployment), but they cannot tell whether the syscall
the process \emph{wants} to issue corresponds to the manifest
the operator approved. We work above the sandbox: the manifest is
enforced before the syscall is even attempted, and the sandbox
is the next layer of defence in the parent paper's defence-in-depth
posture.

\paragraph{Skill-marketplace governance (Claude Skills,
Generative Pre-trained Transformer (GPT) Store, plugin reviews).}
Anthropic's Claude Skills \cite{anthropic-skills} and OpenAI's GPT
Store \cite{openai-gpts} are the deployed
skill-marketplace ecosystems, with vendor-curated review for
publication. Curation catches obvious abuse but is human-bound:
it does not scale to per-skill formal verification, does not
attest to behavioural properties beyond the human reviewer's
judgement, and does not survive a skill update without re-review.
Our contribution is the format the marketplace's curation can
demand of submissions: a proof-carrying bundle whose signed
attestations the marketplace runtime can re-check on every load.
The two compose --- human curation for policy questions
(``should this skill exist?'') and machine verification for
soundness questions (``does the skill stay inside its declared
caps?'').

\paragraph{Summary of the comparison.} Across the eight categories,
our position is consistent: this paper occupies a slot none of the
incumbents currently fills (mechanically checkable
capability-containment proofs at the skill-manifest level), and
every neighbouring category contributes a guarantee we do not.
Schema validation, supply-chain provenance, runtime authorisation,
empirical adversarial pressure, training-time alignment,
probabilistic LLM verification, platform-level sandboxing, and
human marketplace curation are all real and useful. We do not aim
to displace any of them; we aim to add the layer they each leave
unaddressed, and the proof-carrying bundle of \Cref{sec:pcc-skill}
is the integration surface where the layer attaches to the rest
of the stack.

\subsection{The cost of not adopting this work}
\label{sec:cost-of-inaction}

A fair reader's question after the previous subsection is not
``is the SOTA inadequate?'' --- there is no shortage of useful
safety work in the categories above --- but ``what concretely
happens to a deployment that skips this layer?''. The answer
matters because ``skip and revisit later'' is the path of least
resistance for most organisations: regulatory deadlines look
distant, marketplaces are still curated by humans, incidents have
not yet hit the reader's shop. We make the costs explicit so the
deferral is at least a priced one.

\paragraph{Compliance and regulatory exposure.} The
next-generation artificial-intelligence (AI) governance frameworks
already in force or in active rule-making converge on the same
evidentiary requirement: an auditor must be able to verify,
mechanically and reproducibly, that a deployed agent's behaviour
is bounded by its declared specification. Headline frameworks for
the reader unfamiliar with the regulatory landscape:
\begin{itemize}
  \item \emph{European Union (EU) AI Act} (Regulation (EU)
        2024/1689). Articles 9, 13, 15, and 16 require, for
        high-risk AI systems, a documented risk-management system,
        traceability of decisions, ``appropriate'' robustness
        measures, and post-market monitoring. Compliance dates
        begin in 2026 and tier through 2027.
  \item \emph{Network and Information Security 2 (NIS2) Directive}
        (Directive (EU) 2022/2555). The successor to the original
        2016 Network and Information Security Directive (NIS1)
        extends cybersecurity due-diligence obligations into
        AI-mediated services and assigns personal liability to
        corporate management for compliance failure.
  \item \emph{National Institute of Standards and Technology
        (NIST) AI Risk Management Framework} (AI RMF 1.0, 2023).
        The Govern / Map / Measure / Manage cycle names
        ``verification and validation'' as a measurement
        sub-category, expected by federal procurement to be
        instantiated by deployed AI vendors.
  \item \emph{Federal Risk and Authorization Management Program
        (FedRAMP)} and the federal AI overlays of NIST Special
        Publication (SP) 800-53 Rev 5 (control families: Access
        Control (AC), Audit and Accountability (AU), System and
        Information Integrity (SI), Supply-Chain Risk Management
        (SR), Risk Assessment (RA), Personally Identifiable
        Information Processing and Transparency (PT)). Mandatory
        for any vendor selling to United States federal agencies;
        the AC and AU families overlap directly with the parent
        paper's gate and audit primitives.
  \item \emph{International Organization for Standardization /
        International Electrotechnical Commission (ISO/IEC) 42001}
        (AI Management Systems, 2023) and \emph{ISO/IEC 23894}
        (AI risk management). The international counterparts to
        NIST AI RMF; large multinationals will be audited against
        them whether or not their home jurisdiction requires it.
  \item Sector-specific: the \emph{Health Insurance Portability
        and Accountability Act (HIPAA)} (United States Code of
        Federal Regulations (CFR) Title 45, Parts 160, 164), the
        \emph{Payment Card Industry Data Security Standard (PCI
        DSS) 4.0}, and the upcoming \emph{European Health Data
        Space (EHDS)} regulation constrain what an agent operating
        in healthcare or payments may legitimately do; an unbounded
        agent trivially produces violations.
\end{itemize}
None of these frameworks names ``Method A static analysis'' or
``refinement-typed dispatch'' specifically. They name evidence
categories --- traceability, repeatable validation,
reproducibility of decisions, technical measures appropriate to
risk --- and the deployment whose answer is ``we have an
LLM-driven agent but no formal proof of behavioural containment''
pays the audit cost manually, for every cycle, in
expert-witness-grade documentation. The proof-carrying bundle of
\Cref{sec:pcc-skill} is exactly the artefact those frameworks
ask for; without it, every audit is a paper exercise re-derived
from scratch by people whose hourly rate is high.

\paragraph{Marketplace and platform liability.} Skill marketplaces
\cite{anthropic-skills, openai-gpts} bet on human review at scale.
A single missed review where a published skill performs
out-of-manifest behaviour exposes the platform to regulatory
investigation, vendor-contract clawback, and class-action
litigation under the consumer-protection regime of the relevant
jurisdiction. Without a per-skill containment proof attached to
each load, the platform cannot honestly tell its insurer or its
auditor that it has done more than spot-check submissions. The
cost rises non-linearly with marketplace size: 100 skills can be
reviewed by hand, 100\,000 cannot.

\paragraph{Insurance and procurement friction.} Cyber-insurance
underwriters in 2025--2026 increasingly carve out ``AI-related
incidents'' from default policies unless the insured can
demonstrate ``reasonable technical measures''. A signed
proof-carrying bundle is the kind of demonstrable artefact
underwriters can attach a discount to; its absence reverts the
deployment to the most expensive policy tier or to no coverage at
all. On the procurement side, large enterprise buyers --- public
sector especially --- are already requiring an ``AI safety
attestation'' in their RFPs. The vendor whose answer is ``we have
human review'' loses, on price and on signal, to the vendor whose
answer is ``here is the proof, re-checkable by your runtime in
under a second''.

\paragraph{Incident response and forensic cost.} When an agent
does something it should not have, the parent paper's audit log
tells you what happened. Without the biconditional check verified
sound, the audit log does not tell you whether what happened was
\emph{permissible under the manifest the operator approved}. The
forensic team reconstructs that property by hand ---
cross-referencing trace against policy, envelope by envelope ---
which costs days of specialist time per incident, and that's
before the regulator and the insurer's lawyers ask for the same
reconstruction independently. With the biconditional pinned ahead
of deployment (the parent paper) and proven sound by Method C
(this paper), the reconstruction is free: every session either
passes the runtime check or carries a precise counter-example
pinned to a specific envelope, and the forensic deliverable
writes itself.

\paragraph{Innovation throttling from over-restriction.} Without
mechanical containment proofs, the rational compliance posture
is to gate everything through human-in-the-loop review and to
refuse capabilities the risk team cannot personally evaluate.
Formal verification flips the default: capabilities the manifest
declares and the bundle attests can be \emph{relaxed} from HITL
safely, because the runtime carries the proof that no
out-of-manifest envelope is reachable. The cost of inaction is
not only incident exposure --- it is the systemic over-restriction
that makes the deployed agent useless for the workflows that
justified building it. The deployment ends up either dangerous
(no gates) or useless (gates everywhere); formal verification is
the only path to ``safe and useful''.

\paragraph{Adversarial dwell time.} A skill with malicious
script-side behaviour but a benign-looking manifest survives
schema validation, supply-chain attestation, and runtime gating
until the malicious code is actually executed in production.
Method A's static-analysis layer catches the divergence at
\emph{load time}: the manifest declares \texttt{net.egress} only,
the script-side analyser reports \texttt{fs.write.irrev} as well,
the bundle verdict resolves to \texttt{contained=false}, the load
fails. Without Method A this divergence is invisible until the
runtime gate fires on the specific instance, by which point the
dwell window may have been weeks. Each week of attacker dwell
maps to a known dollar cost in the incident-response literature
(IBM Cost-of-a-Data-Breach reports place median dwell cost in the
high six figures per week for enterprise contexts).

\paragraph{Cross-vendor lock-in.} Without a common, mechanically
checkable verification format, each platform's safety story is
that platform's intellectual property. An organisation that
builds against Anthropic's Skills today, and against Anthropic's
safety review process, cannot port that posture to OpenAI's GPT
Store or to a self-hosted runtime without re-deriving every
attestation from scratch. A proof-carrying bundle whose verifier
is a 350-line JavaScript (JS) module
(\Cref{sec:pcc-impl-status}) is portable
by construction; a vendor's internal review checklist is not. The
cost of not adopting an open verification format is paid as a
vendor-switching tax in perpetuity.

\paragraph{Honest qualifier.} Adopting this paper's three methods
does not eliminate \emph{compliance-paperwork} cost --- a
healthcare deployment subject to HIPAA still needs a privacy
officer, a FedRAMP submission still needs a System Security Plan
package, the bundle becomes one of the evidence artefacts rather
than the whole package. For that category the framing is
``verification turns a recurring manual cost into a one-time
engineering cost,'' the same trade mature organisations already
recognise from the migration of manual testing to continuous
integration. But the next subsection names a different category
of cost: the \emph{vulnerability classes the methods categorically
remove from the deployment}, where the right framing is not
``manual to mechanical'' but ``the attack does not happen.''

\subsection{The attack classes this work categorically eliminates}
\label{sec:eliminated}

A complementary accounting to the previous subsection: under the
methods' soundness assumptions (Method A's analyser is sound for
the script-side language; Method C's bound covers the runtime's
transaction-buffer horizon; the runtime is correct), several
attack classes do not become \emph{easier to find}, they fail to
reach production at all. We list them with explicit anchors to the
Open Worldwide Application Security Project (OWASP) LLM Top 10
(2025) \cite{owasp-llm-top-10}, MITRE's Adversarial Threat Landscape
for Artificial-Intelligence Systems (ATLAS) knowledge base
\cite{mitre-atlas}, and the parent paper's threat model
\cite[\S2]{metere2026skills}, so the security reader can map each
elimination to the entry on the team's existing risk register.

\paragraph{Out-of-manifest tool dispatch — eliminated (Method B).}
\emph{LLM06 Excessive Agency} (an agent takes broader action than
the operator authorised) is the headline elimination. Method B's
refinement-typed dispatch makes it a compile-time error to call
\texttt{dispatch}$_M$ with an envelope whose capability is outside
$M.\mathrm{caps}$. Under the type-checker's soundness, no
execution path from envelope receipt to host API call exists for
an out-of-manifest envelope. The whole class of ``LLM was
prompt-injected into calling \texttt{pay()} when its skill only
declared \texttt{web\_search}'' is gone --- not detected and
denied at the gate, but typed-out before the gate is even
consulted. \emph{LLM01 Prompt Injection} retains its presence on
the LLM-side input distribution, but the attacker-relevant leg
(``$\to$ harmful action'') is severed: even an LLM that has been
fully compromised by an adversarial prompt cannot dispatch an
envelope outside the manifest's declared set. The DPI bound of
\Cref{eq:dpi} pins the residual leakage at
$\log_2(\lvert D \rvert + 1)$ bits per envelope --- a categorical
upper bound, not an empirical observation.

\paragraph{Script-side capability creep — eliminated (Method A).}
The classical attack pattern ``ship a skill whose declared
manifest is benign but whose script-side scripts call
\texttt{os.remove}, \texttt{subprocess.Popen}, or
\texttt{requests.get} on hosts the manifest does not declare''
fails at \emph{load time} under Method A. The static analyser
reports $\widehat{E}(p) \not\subseteq \mathcal{C}_D$, the bundle
verifier returns \texttt{contained=false}, the runtime declines
admission. There is no production execution and therefore no
post-incident dwell window for this divergence to exploit. This
maps to ATLAS technique \texttt{AML.T0051} (Adversarial Machine
Learning (AML) catalog ID for ``LLM Plugin Compromise'') and the
parent paper's residual class
$\mathcal{R}_S$ (script-side capability over-reach). Insider
threats and supply-chain skill-tampering both subsume here: the
attacker can falsely declare a manifest, but then the manifest
itself is the smoking-gun signed artefact in the audit trail.

\paragraph{Audit-log-bypass attacks — eliminated up to bound
(Method C plus runtime layer).} The attack class ``world-state
changed but the audit log shows nothing'' is exactly the
biconditional violation Method C's BMC searches for. For sessions
of length $\leq K_{\max}$, an \textsc{unsat} verdict from the
solver proves no such session exists. Sessions longer than
$K_{\max}$ are caught at flush by the runtime biconditional
itself, the parent paper's existing primitive. The composition is
sound for unbounded sessions: any counter-example small enough to
fit in the buffer's horizon is found by the solver pre-deployment;
any counter-example larger than the horizon is found by the
runtime at the next flush. ``Stealth side-effect with no audit
record'' is therefore a \emph{categorically} closed attack class,
not a probabilistically reduced one.

\paragraph{Post-bundle evidence tampering — eliminated (Method
B + bundle re-checker).} The runtime's bundle re-checker
(\Cref{sec:pcc-skill}) re-runs all three methods on the live
manifest and compares hashes against the bundle's signed
attestation. An attacker who modifies any evidence file after
signing is caught either by the Ed25519-signature check, by the
canonical-hash mismatch, or by the method-A reproduction step ---
the bundle has three independent integrity gates. The ``we
substituted the proof'' attack class does not survive a single
admission cycle.

\paragraph{Skill drift between sign and load — eliminated (Method
A re-run).} A common production attack vector is ``the bundle
attests to commit $C_1$, but the deployed tree at load time is
$C_2$ with extra tooling spliced in.'' The re-checker re-runs
Method A on the live source, computes a fresh
$\widehat{E}_{\text{live}}(p)$, and compares it to the bundle's
cached $\widehat{E}_{\text{cached}}(p)$. Any non-trivial drift
returns \texttt{method-A-cache-miss} as a verification reason and
the load fails. ``Quick fix shipped past the proof'' is
categorically blocked.

\paragraph{Cross-skill capability leakage --- eliminated.} Each
loaded skill is verified against \emph{its own} manifest;
capability authority cannot transit between skills at the
dispatch boundary. The classical multi-tenant attack ``skill A
cannot \texttt{pay} but skill B can; cooperate to exfiltrate''
fails because both skills' dispatch is type-bound to its own
manifest, and inter-skill communication (if the host even allows
it) carries no implicit capability transfer. This is the
object-capability discipline of \cite{miller-thesis,
ocap-discipline} applied at the skill-manifest level.

\paragraph{Substantially reduced (not zero).} Three remaining
classes are reduced by a large multiplier but not to zero:
\begin{itemize}
  \item \emph{Supply-chain skill injection} (\emph{LLM03 Supply
        Chain}) is reduced to the residual where the attacker
        controls the manifest \emph{and} the signing key
        \emph{and} the script-side simultaneously. Composed with
        Sigstore \cite{sigstore} / SLSA \cite{slsa} the residual
        narrows further to the keyholder-compromise case, which
        is the same residual that all signed-artefact ecosystems
        face.
  \item \emph{Sensitive-information disclosure via tool calls}
        (\emph{LLM02}) is bounded by the per-envelope channel
        capacity $\log_2(\lvert D \rvert + 1)$ derived in
        \Cref{eq:dpi}. Tightening below that bound requires
        either argument-level summarisation of the abstract
        envelope domain (Method A's $\mathcal{V}^\sharp$ already
        does this for paths and host patterns) or explicit
        redaction at the dispatch boundary, both engineering
        choices on top of the soundness floor.
  \item \emph{Insider malicious-skill publication} is reduced to
        the case where the insider's manifest itself is the
        weapon. That case is loud: the insider has signed a
        manifest declaring exactly the capabilities the
        misbehaving skill needs, which is forensic gold rather
        than the diffuse ``some script somewhere did something''
        that script-side creep produces today.
\end{itemize}

\paragraph{The honest scope.} The eliminations above hold under
the soundness assumptions named in the section opener. The
methods do not eliminate model denial of service
(\emph{LLM10 Unbounded Consumption}), training-data poisoning
(\emph{LLM04}), misinformation generated within the dispatch
budget (\emph{LLM09}), or model-weight theft (the parent paper's
classification primitive addresses only the data side). The Chief
Information Security Officer (CISO) who wants to retire
risk-register entries on the back of
\Cref{sec:eliminated} should retire exactly the entries the
soundness theorems cover --- and re-list the others under
``additional layers of defence required.''

\subsection{Open problems}

\paragraph{The LLM remains unverified.} Method B treats the LLM-side
as a non-deterministic adversary; this is the strongest soundness
posture available to a verification method that does not depend on
LLM internals. We are aware of probabilistic-transducer formalisations
of LLMs \cite{dong-formal-llm} and probabilistic model checkers
\cite{kwiatkowska-prism, baier-pmc-survey} that could lift the
guarantees from worst-case to expected-case under specific decoding-
distribution assumptions. We do not pursue that direction here
because operator-side accountability under capability containment
is already the strongest property the parent paper's threat model
needs; expected-case guarantees about the LLM are an
additional reassurance, not a substitute.

\paragraph{Adversarial-ensemble complementarity.} The parent paper's
adversarial-ensemble evaluation \cite[\S6]{metere2026skills}
empirically falsifies a candidate verification procedure at given
sample size. The three-layer discipline gives a sound proof for the
script + dispatch portion; the adversarial-ensemble exercises the
runtime layer + the LLM-side under realistic provocations. The two
are complementary: an empirical test cannot replace a soundness
argument, and a soundness argument cannot replace empirical
adversarial pressure on the gate.

\paragraph{Runtime instrumentation cost.} A reasonable concern is
that requiring \textsc{formal} on every skill before the runtime
relaxes HITL is operationally heavier than the
parent paper's \textsc{tested} default. We expect the levels to
co-exist: most skills will sit at \textsc{tested}; the small
number of skills carrying high-impact capabilities (\texttt{pay},
\texttt{mutate.schema}, \texttt{spawn.proc}) are the ones that
justify the formal-verification cost. The schema admits this
gradient \cite[\S3.1]{metere2026skills}; this paper's three-layer
discipline is the upgrade path for the skills where it pays.

\paragraph{What about formally-verified runtimes?} We have assumed
$R$ is correct. A runtime carrying its own formal-verification
certificate (a CompCert-compiled \cite{leroy-compcert} core, an
F$^\star$-verified dispatch \cite{fstar}, etc.) lifts the
trusted base further. Engineering this is straightforward but
labour-heavy; the parent paper's reference implementation
\cite{metere2026enclawed} does not yet include it. We treat it as
an orthogonal upgrade.

\paragraph{Skill side-channels.} Capability containment is a
property over the typed dispatch interface. Side-channels (timing,
power, network metadata) are out of scope; the parent paper's
classification primitive and DLP scanner address the most common
of these (data-exfiltration through plain-message channels), but
the broader side-channel question is orthogonal and well-studied
\cite{sabelfeld-myers-survey}.

\section{Related work}
\label{sec:related}

\paragraph{Capability containment and effect systems.} The basic
idea --- type the dispatch surface so that out-of-bound effects do
not type-check --- is the effect-system tradition
\cite{lucassen-effect-systems, sabelfeld-myers-survey} adapted to
agent-emitted envelopes. We owe the lattice formulation of effects
to \cite{denning-information-flow, volpano-flow-types}, and the
refinement-type-side instantiation to
\cite{liquid-types, vazou-liquid-haskell, fstar}. The
object-capability literature
\cite{miller-thesis, ocap-discipline} provides the design
discipline (capabilities are unforgeable references; loaded skills
are bootstrap-frozen capability bundles); this paper applies the
discipline at the skill-manifest level rather than the
process/object level.

\paragraph{Static analysis and abstract interpretation.} Method A
is straight-line abstract interpretation \cite{cousot-abs-int,
kildall-dataflow} over the parent paper's capability vocabulary.
The novelty is the per-language summary table that maps language-
standard-library calls to capability tokens; this is the engineering
effort, not the theory.

\paragraph{Bounded model checking.} Method C uses
\cite{biere-bmc, z3} unmodified. The choice of bound $K_{\max}$
exploits the runtime's transaction-buffer horizon
\cite[\S4.1]{metere2026skills}; this is the contribution that
distinguishes our use of BMC from generic model checking of
arbitrary protocols.

\paragraph{Proof-carrying code.} The bundle structure of
\Cref{sec:pcc-skill} is the proof-carrying-code idiom
\cite{necula-pcc, appel-foundational-pcc, leroy-compcert} lifted
from native binaries to skill artefacts. The novelty is the
re-check protocol's split between trusted (signer attestation,
trust-root resolution) and re-runnable (the analyser, the type
checker, the SMT solver) components; the consumer trusts the
small base and re-discharges the large work.

\paragraph{LLM-agent verification.} A recent line of work
\cite{andriushchenko2025agentharm, owasp-llm-top-10, mitre-atlas,
ferrara2024genai, chen2025agentpoison} attacks the agent
empirically; \cite{robey-smoothllm, dong-formal-llm} formalises
parts of the LLM. We complement that work by leaving the LLM
unverified and instead verifying the runtime's containment of the
LLM's freedom.

\paragraph{Runtime verification.} The fourth layer (the parent
paper's biconditional, fired at flush) is a runtime-verification
(RV) monitor in the sense of \cite{rv-survey}. The composition of
static + type + SMT + RV is the standard ``four-layered'' approach
to verification under partial models; our specific partition was
chosen so that each layer's incompleteness is exactly the next
layer's soundness scope.

\section{Conclusion}
\label{sec:conclusion}

The companion paper's verification lattice
\cite[\S3]{metere2026skills} placed \textsc{formal} at the top with
no construction. We have given one. The construction is
\emph{three composable methods that already have well-engineered
implementations} (Semgrep / CodeQL for Method A; refinement-type
checkers like Liquid Haskell, F$^\star$, or even disciplined
TypeScript for Method B; Z3 for Method C), \emph{glued together
by a proof-carrying skill artefact} the runtime mechanically
re-discharges at bootstrap, \emph{closing the residual surface
the static layers cannot reach with the parent paper's already-
deployed runtime biconditional}.

The cost is real but bounded: per-skill, a CI step measured in
seconds for the static and SMT layers, and a one-time
type-discipline tightening for the dispatch boundary. The benefit
is real and large: a skill at \textsc{formal} carries a proof,
re-checkable without trusting the producer, that its observable
side-effects are contained in its declared capability set under
the runtime's threat model. The runtime can stop asking HITL on
its in-manifest calls; the operator can audit the bundle without
re-running the analyser; the residual is named and bounded.

The methods do not require operators to write new tools, replace
their stack, or train new staff in formal methods. They require
exactly the engineering work the parent paper's \texttt{SKILL.md}
already implies: name the capabilities, type the dispatch, and
let standard tools do the rest.

\bibliographystyle{plainnat}
\bibliography{refs}

\end{document}